\documentclass{article}

\usepackage{microtype}
\usepackage{graphicx}
\usepackage{subfig}
\usepackage{booktabs} 

\usepackage{hyperref}

\usepackage{amsmath}
\usepackage{amsthm}
\usepackage{amsfonts}
\usepackage{amssymb}
\usepackage{soul}
\usepackage{natbib}
\usepackage{bibunits}
\defaultbibliography{main}
\defaultbibliographystyle{icml2019}
\newcommand{\R}{\mathbb{R}}
\newcommand{\N}{\mathbb{N}}
\newcommand{\prob}{\mathbb{P}}
\newcommand{\domain}{\mathcal{X}}
\newcommand{\expectation}{\mathbb{E}}

\newcommand{\cov}{\textnormal{Cov}}
\newcommand{\ei}{\textrm{EI-CF}}
\newcommand{\picf}{\textrm{PI-CF}}
\newcommand{\eis}{\textrm{EI}}
\newtheorem{theorem_paper}{Theorem}
\newtheorem{proposition_paper}{Proposition}
\newtheorem{theorem}{Theorem}[section]
\newtheorem{lemma}[theorem]{Lemma}
\newtheorem{proposition}[theorem]{Proposition}
\newtheorem{corollary}[theorem]{Corollary}
\newtheorem{definition}[theorem]{Definition}
\newcommand{\pfcomment}[1]{}
\newcommand{\pfsuggestion}[1]{}
\newcommand{\pfoptional}[1]{} %
\newcommand{\racomment}[1]{}
\newcommand{\rasuggestion}[1]{}
\newcommand{\pfedit}[1]{}
\newcommand{\raedit}[1]{}

\usepackage[accepted]{icml2019}

\icmltitlerunning{Bayesian Optimization of Composite Functions}

\begin{document}

\twocolumn[
\icmltitle{Bayesian Optimization of Composite Functions}



\icmlsetsymbol{equal}{*}

\begin{icmlauthorlist}
\icmlauthor{Raul Astudillo}{cornell}
\icmlauthor{Peter I. Frazier}{cornell,uber}

\end{icmlauthorlist}
\icmlaffiliation{cornell}{School of Operations Research and Information Engineering, Cornell University, Ithaca, NY, USA}
\icmlaffiliation{uber}{Uber, San Francisco, CA, USA}

\icmlcorrespondingauthor{Raul Astudillo}{ra598@cornell.edu}
\icmlcorrespondingauthor{Peter I. Frazier}{pf98@cornell.edu}

\icmlkeywords{Machine Learning, ICML}

\vskip 0.3in
]



\printAffiliationsAndNotice{}  
\begin{abstract}
We consider optimization of {\it composite} objective functions, i.e., of the form $f(x)=g(h(x))$, where $h$ is a black-box derivative-free expensive-to-evaluate function with vector-valued outputs, and $g$ is a cheap-to-evaluate real-valued function. While these problems can be solved with standard Bayesian optimization, we propose a novel approach that exploits the composite structure of the objective function to substantially improve sampling efficiency. Our approach models $h$ using a multi-output Gaussian process and chooses where to sample using the expected improvement evaluated on the implied non-Gaussian posterior on $f$, which we call expected improvement for composite functions (\ei).  Although \ei\ cannot be computed in closed form, we provide a novel stochastic gradient estimator that allows its efficient maximization.  We also show that our approach is asymptotically consistent, i.e., that it recovers a globally optimal solution as sampling effort grows to infinity, generalizing previous convergence results for classical expected improvement.
Numerical experiments show that our approach dramatically outperforms standard Bayesian optimization benchmarks, reducing simple regret by several orders of magnitude.
\end{abstract}

\begin{bibunit}
\section{Introduction}
We consider optimization of \textit{composite} objective functions, i.e., of the form $f(x) =g(h(x))$, where $h$ is a black-box  expensive-to-evaluate vector-valued function, and $g$ is a real-valued function that can be cheaply evaluated.  We assume evaluations are noise-free. These problems arise, for example, in calibration of simulators to real-world data \cite{vrugt2001calibration, cullick2006improved, schultz2018bayesian};
in materials and drug design \cite{kapetanovic2008computer, frazier2016bayesian} when seeking to design a compound with a particular set of physical or chemical properties;
when finding maximum {\it a posteriori} estimators with expensive-to-evaluate likelihoods \cite{bliznyuk2008bayesian}; and in constrained optimization \cite{gardner14, hernandez_constrained} when seeking to maximize one expensive-to-evaluate quantity subject to constraints on others (See Section \ref{sec:related_work} for a more detailed description of these problems.).

One may ignore the composite structure of the objective and solve such problems using Bayesian optimization (BO) \citep{brochu2010tutorial}, which has been shown to perform well compared with other general-purpose optimization methods for black-box derivative-free expensive-to-evaluate objectives \citep{snoek2012practical}. 
In the standard BO approach, one would build a Gaussian process (GP) prior over $f$ based on past observations of $f(x)$, and then choose points at which to evaluate $f$ by maximizing an acquisition function computed from the posterior.  This approach would not use observations of $h(x)$ or knowledge of $g$.

In this paper, we describe a novel BO approach that leverages the structure of composite objectives to optimize them more efficiently. This approach builds a multi-output GP on $h$, and uses the expected improvement \cite{jones1998efficient} under the implied statistical model on $f$ as its acquisition function.  This implied statistical model is typically non-Gaussian when $g$ is non-linear.
We refer to the resulting acquisition function as expected improvement for composite functions ($\ei$) to distinguish it from the classical expected improvement (EI) acquisition function evaluated on a GP posterior on $f$.

Intuitively, the above approach can substantially outperform standard BO when 
observations of $h(x)$ provide information relevant to optimization that is not available from observations of $f(x)$ alone.
As one example, suppose $x$ and $h(x)$ are both one-dimensional and $g(y)=y^2$.  If $h$ is continuous, $h(0)<0$, and $h(1)>0$, 
 then our approach knows that there  is a global minimum in the interval $(0,1)$, while the standard approach does not.
This informational benefit is compounded further when $h$ is vector-valued.

While $\ei$ is simply the expected improvement under a different statistical model, unlike the classical EI acquisition function, it lacks a closed-form analytic expression and must be evaluated through simulation. We provide a simulation-based method for computing unbiased estimators of the gradient of the $\ei$ acquisition function, which we use within multi-start stochastic gradient ascent to allow efficient maximization. 
We also show that optimizing using $\ei$ is asymptotically consistent under suitable regularity conditions, in the sense that the best point found converges to the global maximum of $f$ as the number of samples grows to infinity.

In numerical experiments comparing with standard BO benchmarks, \ei\ provides immediate regret that is several orders of magnitude smaller, and reaches their final solution quality using less than 1/4 the function evaluations.

\section{Related Work}
\label{sec:related_work}
\subsection{Related Methodological Literature}
We work within the Bayesian optimization framework, whose origins date back to the seminal work of \citet{movckus1975bayesian}, and which has recently become popular due to its success in hyperparameter optimization of machine learning algorithms \citep{snoek2012practical, swersky2013multi}.

Optimizing composite functions has been studied in first- and second-order optimization \citep{shapiro2003class, drusvyatskiy2016efficiency}.  This literature differs from our paper in that it assumes derivatives are available, and also often assumes convexity and that evaluations are inexpensive. In this setting, leveraging the structure of the objective has been found to improve performance, just as we find here in the setting of derivative-free optimization.
However, to the best of our knowledge, ours is the first paper to study composite objective functions within the BO framework and also the first within the more general context of optimization of black-box derivative-free expensive-to-evaluate functions.

Our work is related to  \citet{marque2017efficient}, which proposes a framework for efficient sequential experimental design for GP-based
modeling of nested computer codes. In contrast with our work, that work's goal is not to optimize a composite function, but instead to learn it as accurately as possible within a limited evaluation budget.  A predictive variance minimization sampling policy is proposed and methods for efficient computation are provided. Moreover, it is assumed that both the inner ($h$) and outer ($g$) functions are real-valued and  expensive-to-evaluate black-box functions, while our method uses the ease-of-evaluation of the outer function for substantial benefit.

Our work is also similar in spirit to \citet{overstall2013strategy}, which proposes to model an expensive-to-evaluate vector of parameters of a posterior probability density function using a multi-output GP instead of the function directly using a single-output GP. The surrogate model is then used to perform Bayesian inference.

Constrained optimization is a special case of optimization of a composite objective.  To see this, take $h_1$ to be the objective to be maximized and take $h_i$, for $i>1$, to be the constraints, all of which are constrained to be non-negative without loss of generality. Then, we recover the original constrained optimization problem by setting 
\begin{equation*}
    g(y) = \begin{cases}
y_1,\textnormal{ if } y_i\ge 0 \textnormal{ for all } i>1,\\
-\infty, \textnormal{ otherwise.}
\end{cases}
\end{equation*}
Moreover, when specialized to this particular setting, our $\ei$ acquisition function is equivalent to the expected improvement for constrained optimization as proposed by \citet{schonlau1998global} and \citet{gardner14}. 

Within the constrained BO literature, our work also shares several methodological similarities with \citet{picheny2016bayesian}, which considers an augmented Lagrangian and models its components as GPs instead of it directly as a GP. As in our work, the expected improvement under this statistical model is used as acquisition function. Moreover, it is shown that this approach outperforms the standard BO approach.

Our method for optimizing the $\ei$ acquisition function uses an unbiased estimator of the gradient of $\ei$ within a multistart stochastic gradient ascent framework.  This technique is structurally similar to methods developed for optimizing acquisition functions in other BO settings without composite objectives, including the parallel expected improvement 
\cite{wang2016parallel}
and the parallel knowledge-gradient \cite{wu2016parallel}.

\subsection{Related Application Literature}
Optimization of composite black-box derivative-free expensive-to-evaluate functions arises in a number of application settings in the literature, though this literature does not leverage the composite structure of the objective to optimize it more efficiently.

In materials design, it arises when the objective is the combination of multiple material properties via a {\it performance index} \cite{ashby1993materials}, 
e.g., the specific stiffness, which is the ratio of Young's modulus and the density, or {\it normalization} \cite{jahan2015state}. 
Here, $h(x)$ is the set of material properties that results from a particular chemical composition or set of processing conditions, $x$, and $g$ is given by the performance index or normalization method used.  Evaluating the material properties, $h(x)$, that result from a materials design typically requires running expensive physical or computational experiments that do not provide derivative information, for which BO is appropriate \citep{kapetanovic2008computer,ueno2016combo,ju2017designing,griffiths2017constrained}.

Optimization of composite functions also arises
in calibration of expensive black-box simulators  \citep{vrugt2001calibration, cullick2006improved, schultz2018bayesian},
where the goal is to find input parameters, $x$, to the simulator such that its vector-valued output, $h(x)$, most closely matches a vector data observed in the real world, $y_{\textnormal{obs}}$. Here, the objective to be minimized is $g(h(x)) = || h(x) - y_{\textnormal{obs}} ||$, where $||\cdot||$ is often the $L_1$ norm, $L_2$  norm, or some monotonic transformation of the likelihood of observation errors.


If one has a prior probability density $p$ on $x$, and the log-likelihood of some real-world observation error, $\epsilon$, is proportional to $|| \epsilon ||$ (as it would be, for example, with independent normally distributed errors taking $||\cdot||$ to be the $L_2$ norm),
then, finding the maximum {\it a posteriori} estimator of $x$ 
\citep{bliznyuk2008bayesian}
is an optimization problem with a composite objective: the log-posterior is equal to the sum of a constant and $g(h(x)) = -\beta|| h(x) - y_{\textnormal{obs}} ||^2 + \log(p(x))$ (In this example, $g$ is actually a function of both $h(x)$ and $x$.  Our framework extends easily to this setting as long as $g$ remains a cheap-to-evaluate function.).



\section{Problem Description and Standard Approach}
As described above, we consider optimization of objectives of the form $f(x) = g(h(x))$, where 
$h:\domain\rightarrow\mathbb{R}^{m}$ is a black-box expensive-to-evaluate continuous function whose evaluations do not provide derivatives, $g:\mathbb{R}^{m}\rightarrow\mathbb{R}$ is a function that can be cheaply evaluated, and $\domain\subset\R^d$.   
As is common in BO, we assume that $d$ is not too large ($<20$) and that  projections onto $\domain$ can be efficiently computed.
We also place the technical condition that $\expectation \left[|g(Z)|\right] < \infty$, where $Z$ is an $m$-variate standard normal random vector. The problem to be solved is
\begin{align}
\label{main}
    \max_{x\in\domain}g(h(x)).
\end{align}
As discussed before, one can solve problem \eqref{main} by applying standard BO to the objective function, $f := g\circ h$. This approach models $f$ as drawn from a GP prior probability distribution.  Then, iteratively, indexed by $n$, this approach would choose the point $x_n \in \domain$ at which to evaluate $f$ next by optimizing an acquisition function, such as the EI acquisition function \citep{movckus1975bayesian,jones1998efficient}.  This acquisition function would be calculated from the posterior distribution given $\{\left(x_i, f(x_i)\right)\}_{i=1}^n$, which is itself a GP, and would quantify the value of an evaluation at a particular point.  
Although $h(x)$ would be observed as part of this standard approach, these evaluations would be ignored when calculating the posterior distribution and acquisition function.

\section{Our Approach}

We now describe our approach, which like the standard BO approach is comprised of a statistical model and an acquisition function. Unlike standard BO, however, our approach leverages the additional information in evaluations of $h$, along with knowledge of $g$.  We argue below and demonstrate in our numerical experiments that this additional information can substantially reduce the number of evaluations required to find good approximate global optima.

Briefly, our statistical model is a multi-output Gaussian process on $h$ \citep{alvarez2012kernels} (Section~\ref{sec:GP}), and our acquisition function, $\ei$, is the expected improvement under this statistical model (Section~\ref{sec:EI}).
This acquisition function, unfortunately, cannot be computed in closed form for most functions $g$.  
In Section~\ref{sec:computation}, under mild regularity conditions, we provide a technique for efficiently maximizing $\ei$.  
We also provide a theoretical analysis showing that $\ei$ is asymptotically consistent (Section~\ref{sec:theory}).
Finally, we conclude this section by discussing the computational complexity of our approach (Section~\ref{sec:complexity}).

\subsection{Statistical Model}
\label{sec:GP}
We model $h$ as drawn from a multi-output GP distribution \citep{alvarez2012kernels}, $\mathcal{GP}(\mu,K)$, where $\mu:\domain\rightarrow\R^{m}$ is the mean function, $K:\domain\times\domain\rightarrow\mathcal{S}_{++}^{m}$ is the covariance function, and $\mathcal{S}_{++}^{m}$ is the cone of positive definite matrices. Analogously to the single-output case, after observing $n$ evaluations of $h$, $h(x_1),\ldots,h(x_n)$, the posterior distribution on $h$ is again a multi-output GP, $\mathcal{GP}(\mu_n,K_n)$, where $\mu_n$ and $K_n$ can be computed in closed form in terms of $\mu$ and $K$ \citep{liu2018remarks}.

\subsection{Expected Improvement for Composite Functions}
\label{sec:EI}

We define the expected improvement for composite functions analogously to the classical expected improvement, but where our posterior on $f(x)$ is given by the composition of $g$ and the normally distributed posterior distribution on $h(x)$:
\begin{align}
\ei_n(x) = \expectation_{n}\left[\left\{g(h(x))-f_n^{*}\right\}^{+}\right],
\end{align}
where $f_n^* = \max_{i=1,
\ldots,n}f(x_i)$ is the maximum value across the points that have been evaluated so far, $x_1,\ldots, x_n$,  $\expectation_n$ indicates the conditional expectation given the available observations at time $n$,  $\{\left(x_i, h(x_i)\right)\}_{i=1}^n$, and $a^+ = \max(0,a)$ is the positive part function.

When $h$ is scalar-valued and $g$ is the identity function, $\ei_n$ is equivalent to the classical expected improvement computed directly from a GP prior on $f$, and has an analytic expression that makes it easy to compute and optimize.   For general nonlinear functions $g$, however, $\ei_n$ cannot be computed in closed form. Despite this, as we shall see next, under mild regularity conditions, $\ei_n$ can be efficiently maximized.

\begin{figure*}
\subfloat[Posterior on $h$ used by our $\ei$ acquisition function]{%
\includegraphics[width=0.49\linewidth]{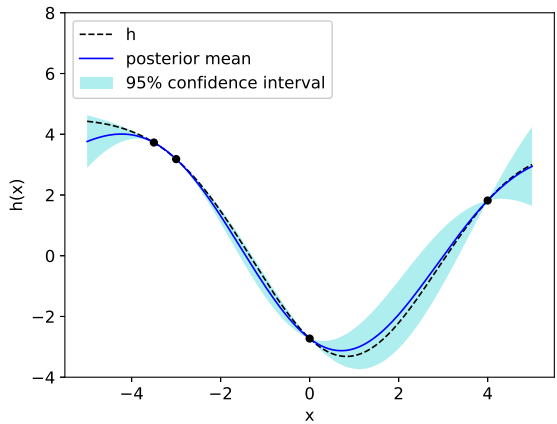}
}\hfill\\
\subfloat[Implied posterior on $f$ used by our $\ei$ acquisition function]{%
\includegraphics[width=0.49\linewidth]{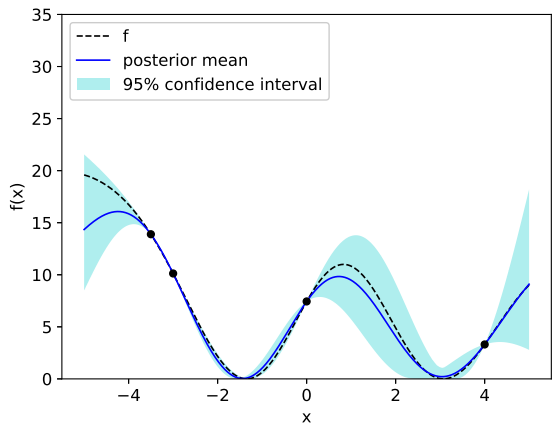}
}
\subfloat[Posterior on $f$ used by the classical EI acquisition function]{%
\includegraphics[width=0.49\linewidth]{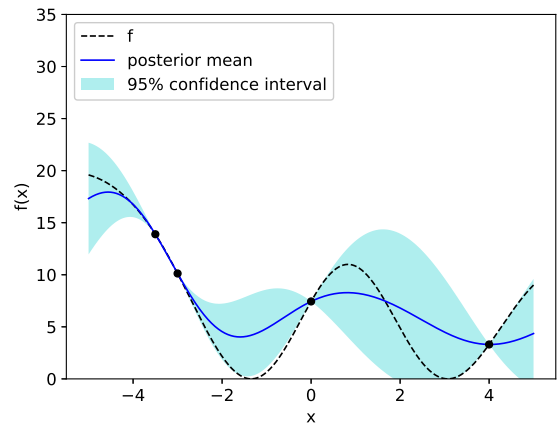}
}\\
\subfloat[$\ei$ acquisition function]{%
\includegraphics[width=0.49\linewidth]{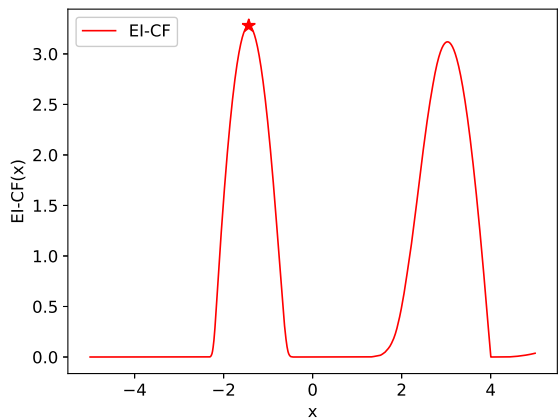}
}
\subfloat[Classical EI acquisition function]{%
\includegraphics[width=0.49\linewidth]{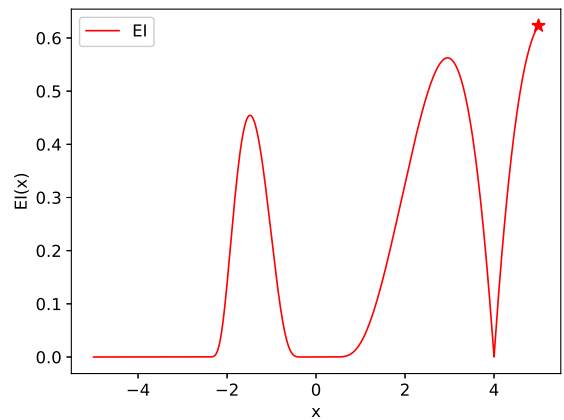}
}
\caption{Illustrative example of the $\ei$ and classical EI acquisition functions, in a problem where $h$ is scalar-valued and $g(h(x))=h(x)^2$.
Observations of $h(x)$ provide a substantially more accurate view of where global optima of $f$ reside as compared with observations of $f(x)$ alone, and cause $\ei$ to evaluate at points much closer to these global optima.
\label{fig:example}}
\end{figure*}

Figure~\ref{fig:example} illustrates the $\ei$ and classical EI acquisition functions in a setting where $h$ is scalar-valued, $f(x) = g(h(x)) = h(x)^2$, we have evaluated $h$ and $f$ at four points, and we wish to minimize $f$.  The right-hand column shows the posterior distribution on $f$ and EI acquisition function using the standard approach: posterior credible intervals have 0 width at points where we have evaluated (since evaluations are free from noise), and become wider as we move away from them.  The classical expected improvement is largest near the right limit of the domain, where the posterior mean computed using observations of $f(x)$ alone is relatively small and has large variance.

The left-hand column shows the posterior distribution on $h$, computed using a GP (single-output in this case, since $h$ is scalar-valued), the resulting posterior distribution on $f$, and the resulting $\ei$ acquisition function.  The posterior distribution on $f(x)$ (which is not normally distributed) has support only on non-negative values, and places higher probability on small values of $f(x)$ in the regions $x\in [-2,-1] \cup [2.5,3.5]$, which creates a larger value for $\ei$ in these regions.  

Examining past observations of $h(x)$, the points with high $\ei$ ($x \in [-2,-1] \cup [2.5,3.5]$) seem substantially more valuable to evaluate than the point with largest EI ($x=5$).  Indeed, for the region $[-2,-1]$, we know that $h(x)$ is below $0$ near the left limit, and is above $0$ near the right limit.  The continuity of $h$ implies that $h(x)$ is 0 at some point in this region, which in turn implies that $f$ has a global optimum in this region.  Similarly, $f$ is also quite likely to have a global optimum in $[2.5,3.5]$.  $\ei$ takes advantage of this knowledge in its sampling decisions while classical EI does not.

\subsection{Computation and Maximization of $\ei$}
\label{sec:computation}
We now describe computation and efficient maximization of $\ei$.
For any fixed $x\in\domain$, the time-$n$ posterior distribution on $h(x)$ is multivariate normal.
(By the ``time-$n$ posterior distribution'', we mean the conditional distribution given $\{\left(x_i,h(x_i)\right)\}_{i=1}^n$.)
We let $\mu_{n}(x)$ denote its mean vector and $K_n(x)$ denote its covariance matrix. We also let $C_{n}(x)$ denote the lower Cholesky factor of $K_n(x)$. Then, we can express $h(x) = \mu_{n}(x) + C_{n}(x)Z$, where $Z$ is a $m$-variate standard normal random vector under the time-$n$ posterior distribution, and thus
\begin{equation}
\label{ei_mc}
    \ei_n(x) = \expectation_{n}\left[\left\{ g(\mu_{n}(x) + C_{n}(x)Z) -  f_{n}^{*}\right\}^+ \right].
\end{equation} 
Thus, we can compute $\ei_n(x)$ via Monte Carlo, as summarized in Algorithm \ref{alg:ei}. We note that \eqref{ei_mc} and the condition $\expectation [|g(Z)|] < \infty$ imply that $\ei_n(x)$ is finite  for all $x\in\domain$.
\begin{algorithm}[h]
\begin{algorithmic}[1]
\caption{Computation of $\ei$}
\label{alg:ei}
\REQUIRE{point to be evaluated, $x$; number of Monte Carlo samples, $L$}
\FOR{$\ell = 1,\ldots,L$}
\STATE{Draw sample $Z^{(\ell)}\sim\mathcal{N}_{m}(0_m,I_m)$ and compute $\alpha^{(\ell)} := \left\{ g\left(\mu_{n}(x) + C_{n}(x)Z^{(\ell)}\right) -  f_{n}^{*}\right\}^+$}
\ENDFOR
\STATE{Estimate $\ei_n(x)$ by $\frac{1}{L}\sum_{\ell=1}^{L}\alpha^{(\ell)}$}
\end{algorithmic}
\end{algorithm}

In principle, the above is enough to maximize $\ei_n$ using a derivative-free global optimization algorithm (for non-expensive noisy functions). However, such methods typically require a large number of samples, and  optimization can be typically performed with much greater efficiency if derivative information is available \cite{jamieson2012query,swisher2000survey}. The following proposition describes a simulation-based procedure for generating such derivative information.
A formal statement and proof can be found in Appendix \ref{app:gradient_eicf}.

\begin{proposition_paper}
\label{thm:ei_gradient_main}
Under mild regularity conditions, 
$\ei_n$ is differentiable almost everywhere, and its gradient, when it exists, is given by
\begin{equation}\label{eqn:ei_gradient}
    \nabla\ei_n(x) = \expectation_{n}\left[\gamma_n(x,Z)\right],
\end{equation}
where
\begin{equation}
\label{eq:gamma}
    \gamma_n(x,Z) = \begin{cases}
0, \textnormal{ if } g(\mu_{n}(x) + C_{n}(x)Z) \leq  f_n^*,\\
\nabla g(\mu_{n}(x) + C_{n}(x)Z), \textnormal{ otherwise.}
\end{cases}
\end{equation}
\end{proposition_paper}

Thus, $\gamma_n$ provides an unbiased estimator of $\nabla \ei_n$.
To construct such an estimator, we would draw an independent standard normal random vector $Z$ and then compute $\gamma_n(x,Z)$ using \eqref{eq:gamma}, optionally averaging across multiple samples as in Algorithm~\ref{alg:ei}.
To optimize $\ei_n$, we then use this gradient estimation procedure within stochastic gradient ascent, using multiple restarts.  The final iterate from each restart is an approximate stationary point of the $\ei_n$.  We then use Algorithm~\ref{alg:ei} to select the stationary point with the best value of $\ei_n$.


\subsection{Theoretical Analysis}
\label{sec:theory}
Here we present two results giving insight into the properties of the expected improvement for composite functions. The first result simply states that, when $g$ is linear, $\ei$ has a closed form analogous to the one of the classical $\eis$. 
\begin{proposition_paper}
\label{prop:eicf_linear_main}
Suppose that $g$ is given by $g(y) = w^\top y$ for some fixed $w\in\R^m$. Then,
\begin{equation*}
\ei_n(x) =  \Delta_n(x)\Phi\left(\frac{\Delta_n(x)}{\sigma_n(x)}\right) + \sigma_n(x)\varphi\left(\frac{\Delta_n(x)}{\sigma_n(x)}\right),
\end{equation*}
where $\Delta_n(x) = w^\top \mu_n(x) - f_n^*$, $\sigma_n(x) = \sqrt{w^\top K_{n}(x)w}$, and $\varphi$ and $\Phi$ are the standard normal probability density function and cumulative distribution function, respectively. 
\end{proposition_paper}

 This result can be easily verified by noting that, since the time-$n$ posterior distribution of $h(x)$ is $m$-variate normal with mean vector  $\mu_n(x)$ and covariance matrix $K_n(x)$, the time-$n$ posterior distribution of $w^\top h(x)$ is normal with mean $w^\top\mu_n(x)$ and variance $w^\top K_n(x)w$. Proposition \ref{prop:eicf_linear_main} does not, however, mean that our approach is equivalent to the classical one when g is linear. This is because, in general, the posterior distribution given observations of $h(x)$ is different from the one given observations of $w^\top h(x)$ .  We refer the reader to Appendix \ref{app:eicf_linear} for a discussion.
  
Our second result states that, under suitable conditions, our acquisition function is asymptotically consistent, i.e., the solution found by our method converges to the global optimum when the number of evaluations goes to infinity. An analogous result for the classical expected improvement was proved by  \citet{vazquez2010convergence}.
\begin{theorem_paper}
\label{main_theorem}
Let $\{x_n\}_{n\in\N}$ be the sequence of evaluated points and suppose there exists $n_0\in\N$ such that for all $n\geq n_0$,
\begin{equation*}
    x_{n+1} \in\arg\max_{x\in\domain}\ei_{n}(x).
\end{equation*}
Then, under suitable regularity conditions and as $n\rightarrow \infty$,
\begin{equation*}
    f_{n}^* \rightarrow \max_{x\in\domain}f(x).
\end{equation*}
\end{theorem_paper}
A formal statement and proof of Theorem \ref{main_theorem} can be found in Appendix \ref{app:eicf_convergence}.
\subsection{Computational Complexity of Posterior Inference}
\label{sec:complexity}

The computation required to maximize the classical $\eis$ acquisition function is dominated by the computation of the posterior mean and variance and thus in principle scales as $\mathcal{O}(n^2)$ (with a pre-computation of complexity $\mathcal{O}(n^3)$) with respect to the number of evaluations \citep{shahriari2016taking}. However, recent advances on approximate fast GP training and prediction may considerably reduce the computational burden \citep{pleiss2018constant}.

In our approach, the computational cost is again dominated by the computation of the posterior mean and covariance matrix, $\mu_n(x)$ and $K_n(x)$, respectively. When the outputs of $h$ are modeled independently, the components of $\mu_n(x)$ and $K_n(x)$ can be computed separately ($K_n(x)$ is diagonal in this case) and thus computation of the posterior mean and covariance
scales as $\mathcal{O}(mn^2)$. This allows our approach to be used even if $h$ has a relatively large number of outputs. However, in general, if correlation between components of $h$ is modeled, these computations scale as $\mathcal{O}(m^2n^2)$. Therefore, in principle there is a trade-off between modeling correlation between components of $h$, which presumably allows for a faster learning of $h$, and retaining tractability in the computation of the acquisition function. 

\section{Numerical Experiments}
We compare the performance of three acquisition functions: expected improvement (EI), probability of improvement (PI) \cite{kushner1964new}, and the acquisition function that chooses points uniformly at random (Random), both under our proposed statistical model and the standard one, i.e., modeling $h$ using a multi-output GP and modeling $f$ directly using a single-output GP, respectively.  We refer the reader to Appendix \ref{app:picf} for a formal definition of the probability of improvement under our statistical model, and a discussion of how we maximize this acquisition function in our numerical experiments. To distinguish each acquisition function under our proposed statistical model from its standard version, we append "-CF" to its abbreviation; so if the classical expected improvement acquisition function is denoted EI, then the expected improvement under our statistical model is denoted EI-CF, as previously defined. We also include as a benchmark the predictive entropy search (PES) acquisition function \citep{hernandez2014predictive} under the standard statistical model, i.e., modeling $f$ directly using a single-output GP. For all problems and methods, an initial stage of evaluations is performed using $2(d+1)$ points chosen uniformly at random over $\domain$.

 For EI-CF, PI-CF, and Random-CF, the outputs of $h$ are
modeled using independent GP prior distributions. All GP
distributions involved, including those used by the standard
BO methods (EI, PI, Random, and PES), have a constant
mean function and ARD squared exponential covariance
function; the associated hyperparameters are estimated
under a Bayesian approach. As proposed in \citet{snoek2012practical}, for all methods we use an averaged version of the acquisition function, obtained by first drawing 10 samples of the GP hyperparameters, computing the acquisition function conditioned on each of these hyperparameters, and then averaging the results.

We calculate each method's simple regret at the point it believes to be the best based on evaluations observed thus far.  We take this point to be the point with the largest (or smallest, if minimizing) posterior mean.  For EI-CF, PI-CF, and Random-CF, we use the posterior mean on $f$ implied by the GP posterior on $h$, and for EI, PI, Random, and PES we use the GP posterior on $f$. Error bars in the plots below show the mean of the base-10 logarithm of the simple regret plus and minus 1.96 times the standard deviation divided by the square root of the number of replications. Each experiment was replicated 100 times.

Our code is available at \citet{bocf}.

\subsection{GP-Generated Test Problems}
The first two problems used functions $h$ generated at random from GPs.  Each component of $h$ was generated by sampling on a uniform grid from independent GP distributions with different fixed hyperparameters and then taking the resulting posterior mean as a proxy; the hyperparameters were not known to any of the algorithms. The details of each problem, including the function $g$ used, are summarized in Table \ref{tab:test}.

Results are shown on a logarithmic scale in Figures~\ref{fig:test1} and \ref{fig:test2}, where the horizontal axis indicates the number of samples following the initial stage.  $\ei$ outperforms the other methods significantly. Regret is smaller than the best of the standard BO benchmarks throughout and by several orders of magnitude after 50 evaluations (5 orders of magnitude smaller in test 1, and 2 in test 2). It also requires substantially fewer evaluations beyond the first stage to reach the regret achieved by the best of the standard BO benchmarks in 100 evaluations: approximately 30 in test 1, and 10 in test 2.  Random-CF performs surprisingly well in type-2 GP-generated problems, suggesting that a substantial part of the benefit provided by our approach is the value of the additional information available from observing $h(x)$.  In type-1 problems it does not perform as well, suggesting that high-quality decisions about where to sample are also important.

\newcommand{\zap}[1]{}

\begin{table}
\begin{center}
 \begin{tabular}{||c c c c||}
 \hline
 Problem & $\domain$ & $g$ & $m$ \\ [0.5ex] 
 \hline\hline
 1 & $[0,1]^4$  & $g(y) = -\|y - y_{\textnormal{obs}}\|_{2}^{2}$ & 5 \\ 
 \hline
 2 & $[0,1]^3$  & $g(y) = -\sum_{j}\exp(y_j)$ & 4 \\
 \hline
\end{tabular}
\caption{Description of GP-generated test problems \label{tab:test}}
\end{center}
\end{table}

\begin{figure}[h]
  \centering
  \includegraphics[width=0.47\textwidth]{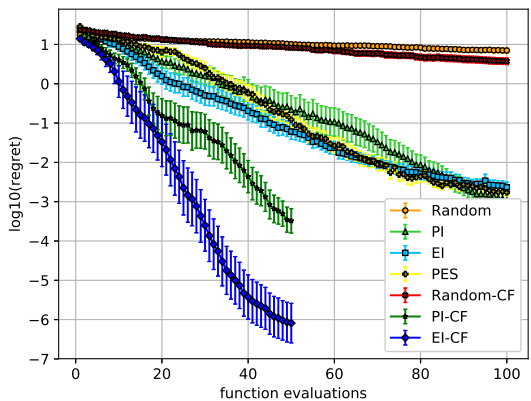}
  \caption{Expected log10(regret) in type-1 GP-generated test problems, estimated from  100 independent replications. 
 These problems use $\domain=[0,1]^4$, $g(y) = -||y - y_{\textnormal{obs}}|_2^2$, and $m=5$. 
 $\ei$ outperforms other methods by a large margin.
  \label{fig:test1}}
  
\end{figure}

\begin{figure}[h]
  \centering
  \includegraphics[width=0.47\textwidth]{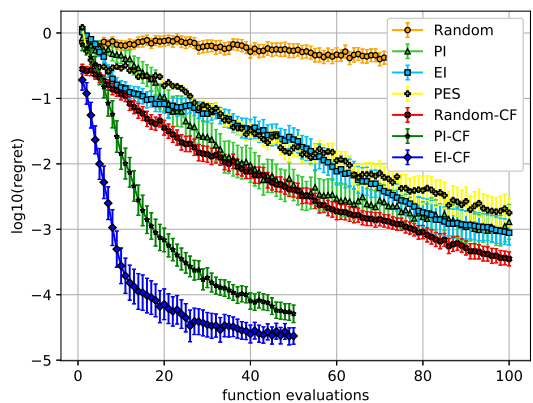}
  \caption{
 Expected log10(regret) in type-2 GP-generated test problems, estimated from  100 independent replications.  These problems use $\domain=[0,1]^3$, $g(y) = -\sum_j \exp(y_j)$, and $m=4$.
  \label{fig:test2}}
\end{figure}

\subsection{Standard Global Optimization Test Problems}
We assess our approach's performance on two standard benchmark functions from the global optimization literature: the Langermann  \citep{langermann} and Rosenbrock \cite{rosenbrock} functions. We refer the reader to Appendix \ref{app:test_functions} for a description of how these functions are adapted to our setting.
 
 Results of applying our method and benchmarks to these problems are shown on a logarithmic scale in Figures~\ref{fig:Langermann}  and ~\ref{fig:rosenbrock}.  As before, $\ei$ outperforms competing methods with respect to the final achieved regret. PI-CF and Random-CF also perform well compared to methods other than $\ei$. Moreover, in the Langermann test problem,  PI-CF outperforms EI-CF during the first 20 evaluations.

\begin{figure}[h]
  \centering
  \includegraphics[width=0.47\textwidth]{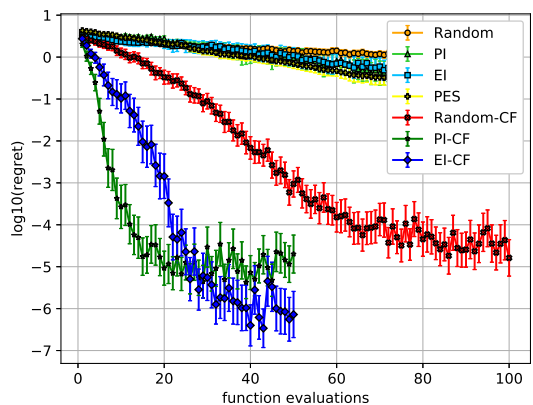}
  \caption{Expected log10(regret) in the Langermann test function, estimated from  100 independent replications.
 \label{fig:Langermann}}
\end{figure}

\begin{figure}[h]
  \centering
  \includegraphics[width=0.47\textwidth]{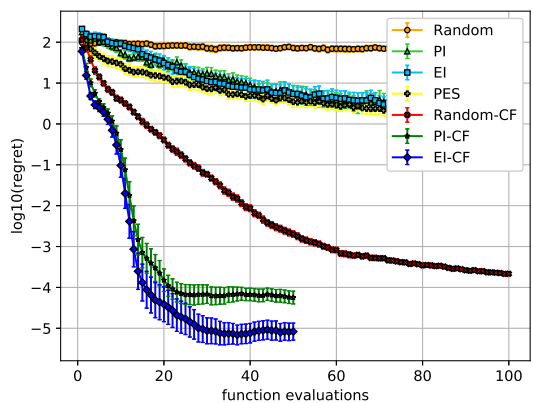}
  \caption{
 Expected log10(regret) in the Rosenbrock test problem, estimated from  100 independent replications.
  \label{fig:rosenbrock}}
\end{figure}
\subsection{Environmental Model Function}
The environmental model function was originally proposed by \citet{bliznyuk2008bayesian} and is now a well-known test problem in the literature of Bayesian calibration of expensive computer models. It models a chemical accident that has caused a pollutant to spill at two locations into a long
and narrow holding channel, and is based on a first-order approach to modeling the concentration of substances in such channels under the assumption that the channel can be approximated by an infinitely long one-dimensional system
with diffusion as the only method of transport. This leads to the concentration
representation
\begin{align*}
    c(s,t;M, D, L, \tau)
    = {} & \frac{M}{\sqrt{4\pi Dt}}\exp\left(\frac{-s^2}{4 Dt}\right) +\\
   & \frac{\mathbb{I}\{t > \tau\}M}{\sqrt{4\pi D(t-\tau)}}\exp\left(\frac{-(s-L)^2}{4 D(t-\tau)}\right),
\end{align*}
where $M$ is the mass of pollutant spilled at each location, $D$ is the diffusion rate in the channel, $L$ is the location of the second spill,  and $\tau$ is the time of the second spill. 

We observe $c(s,t;M_0, D_0, L_0, \tau_0)$ in a $3\times 4$ grid of values; specifically, we observe $\{c(s,t;M_0, D_0, L_0, \tau_0):(s,t)\in S\times T\}$, where $S=\{0, 1, 2.5\}$, $T = \{15, 30, 45, 60\}$, and $(M_0, D_0, L_0, \tau_0)$  are the underlying true values of these parameters. Since we assume noiseless observations, the calibration problem reduces to finding $(M, D, L, \tau)$ so that the observations at the grid minimize the sum of squared errors, i.e., our goal is to minimize
\begin{equation*}
    \sum_{(s,t)\in S\times T} (c(s,t;M_0, D_0, L_0, \tau_0) - c(s,t;M, D, L, \tau))^2.
\end{equation*}
In our experiment, we take $M_0 = 10$, $D_0 = 0.07$, $L_0 = 1.505$ and $ \tau_0 = 30.1525$. The search domain is $M\in[7,13]$, $D \in [0.02, 0.12]$, $L \in [0.01, 3]$ and $ \tau \in [30.01, 30.295]$.

Results from this experiment are shown in Figure~\ref{fig:test5}.  As above, EI-CF performs best, with PI-CF and Random-CF also significantly outperforming benchmarks that do not leverage the composite structure.

\begin{figure}[h]
  \centering
  \includegraphics[width=0.47\textwidth]{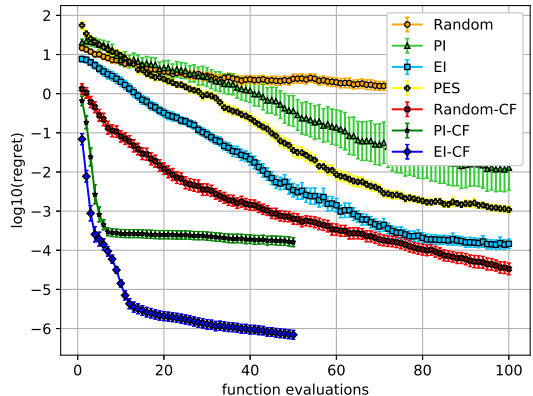}
  \caption{
 Expected log10(regret) in the environmental model function test problem, estimated from 100 independent replications.
  \label{fig:test5}}
\end{figure}


\section{Conclusion and Future Work}
We have proposed a novel Bayesian optimization approach for objective functions of the form $f(x)=g(h(x))$, where $h$ is a black-box expensive-to-evaluate vector-valued function, and $g$ is a real-valued function that can be cheaply evaluated. Our numerical experiments show that this approach may substantially outperform standard Bayesian optimization, while retaining computational tractability.

There are several relevant directions for future work. Perhaps the most evident is to understand whether other well-known acquisition functions can be generalized to our setting in a computationally tractable way. We believe this to be true for predictive entropy search \citep{hernandez2014predictive} and knowledge gradient \citep{scott2011correlated}. Importantly, these acquisition functions would allow noisy and decoupled evaluations of the components of $h$, thus increasing the applicability of our approach. However, in the standard Bayesian optimization setting, they are already computationally intensive and thus a careful analysis is required to make them computationally tractable in our setting.


\section*{Acknowledgements}
This work was partially supported by NSF CAREER CMMI-1254298, NSF CMMI-1536895 and 
AFOSR FA9550-15-1-0038. The authors also thank Eytan Bakshy for helpful comments.


\putbib
\end{bibunit}



\onecolumn

\setcounter{section}{0}
\appendix
\begin{bibunit}
\section{Unbiased Estimator of the Gradient of $\ei$}
\label{app:gradient_eicf}
In this section we prove that, under mild regularity conditions, 
 $\ei_n$ is differentiable and an unbiased estimator of its gradient can be efficiently computed. More concretely, we prove the following.
 \begin{proposition}
  \label{thm:ei_gradient}
 Suppose that $g$ is differentiable and let $\domain_0$ be an open subset of $\domain$ so that $\mu_n$ and $K_n$ are differentiable on $\domain_0$ and there exists a measurable function $\eta:\R^m\rightarrow \R$ satisfying 
\begin{enumerate}
    \item $\|\nabla g\left(\mu_{n}(x) + C_{n}(x)z\right)\|<\eta(z)$ for all $x\in\domain_0, \ z\in\R^m$,
    \item $\expectation[\eta(Z)]<\infty$, where $Z$ is a $m$-variate standard normal random vector.
\end{enumerate}
Further, suppose that for almost every $z\in\R^m$ the set $\{x\in\domain_0 : g\left(\mu_{n}(x) + C_{n}(x)z\right) = f_n^*\}$ is countable. Then,
$\ei_n$ is differentiable on $\domain_0$ and its gradient is given by
\begin{equation*}
    \nabla\ei_n(x) = \expectation_n\left[\gamma_n(x,Z)\right],
\end{equation*}
where the expectation is with respect to $Z$ and
\begin{equation*}
    \gamma_n(x,z) = \begin{cases}
\nabla g\left(\mu_{n}(x) + C_{n}(x)z\right), \textnormal{ if } g\left(\mu_{n}(x) + C_{n}(x)z\right) >  f_n^* ,\\
0, \textnormal{ otherwise.}
\end{cases}
\end{equation*}
\end{proposition}
\begin{proof}
Since $g$ is differentiable and $\mu_n$ and $K_n$ are differentiable on $\domain_0$,  for any fixed $z\in\R^m$ the function $x \mapsto g\left(\mu_{n}(x) + C_{n}(x)z\right)$ is differentiable on $\domain_0$ as well. This in turn  implies that the function $x \mapsto \{g\left(\mu_{n}(x) + C_{n}(x)z\right) - f_n^*\}^{+}$ is continuous on $\domain_0$ and differentiable at every $x\in\domain_0$ such that $g\left(\mu_{n}(x) + C_{n}(x)z\right) \neq f_n^*$, with gradient equal to $\gamma(x,z)$. From our assumption that for almost every $z\in\R^m$ the set $\{x\in\domain : g\left(\mu_{n}(x) + C_{n}(x)z\right) = f_n^*\}$ is countable, it follows that for almost every $z$ the function $x \mapsto \{g\left(\mu_{n}(x) + C_{n}(x)z\right) - f_n^*\}^{+}$ is continuous on $\domain_0$ and differentiable on all $\domain_0$, except maybe on a countable subset. Using this, along with conditions 1 and 2, and Theorem 1 in \citet{l1990unified}, the desired result follows.
\end{proof}

We end this section by making a few remarks.
\begin{itemize}
    \item If $\mu$ and $K$ are differentiable on $\textnormal{int}(\domain)$, then one can show that $\mu_n$ and $K_n$ are differentiable on $\textnormal{int}(\domain)\setminus\{x_1,\ldots, x_n\}$.
    \item If one imposes the stronger condition $\expectation[\eta(Z)^2]<\infty$, then $\gamma_n$ has finite second moment, and thus this unbiased estimator of $\nabla\ei_n(x)$ can be used within stochastic gradient ascent to find a stationary point of $\ei_n$ \citep{bottou1998online}.
     \item In Proposition \ref{thm:ei_gradient}, the condition that for almost every $z\in\R^m$ the set $\{x\in\domain_0 : g\left(\mu_{n}(x) + C_{n}(x)z\right) = f_n^*\}$ is countable, can be weakened to the following more technical condition: for almost every $z\in\R^m$, every $x\in\domain_0$ and every $i\in\{1,\ldots , d\}$, there exists $\epsilon>0$ such that the set $\{x + h e_i: |h| < \epsilon \textnormal{ and } g\left(\mu_{n}(x + h e_i) + C_{n}(x + h e_i)z\right) = f_n^*\}$ is countable, where $e_i$ denotes the $i$-th canonical vector in $\R^d$.
\end{itemize}

\section{EI-CF and EI Do Not Coincide When $g$ Is Linear}
\label{app:eicf_linear}
Recall the following result that was stated in the main paper.
\begin{proposition}
\label{prop:eicf_linear_app}
Suppose that $g$ is given by $g(y) = w^\top y$ for some fixed $w\in\R^m$. Then,
\begin{equation*}
\ei_n(x) =  \Delta_n(x)\Phi\left(\frac{\Delta_n(x)}{\sigma_n(x)}\right) + \sigma_n(x)\varphi\left(\frac{\Delta_n(x)}{\sigma_n(x)}\right)
\end{equation*}
\end{proposition}

The resemblance of the above expression to the classical EI acquisition function may make one think that, in the above case, $\ei$ coincides, in some sense, with the classical $\eis$ under an appropriate choice of the prior distribution. 

Indeed, suppose that we set a single-output GP prior with mean $w^\top\mu(x)$ and covariance function $w^\top K_n(x)w$ of $f$ (and fix its hyperparameters), then
\begin{align*}
 \expectation\left[\left\{w^\top h(x)-f_n^*\right\}^{+}\mid x_i, w^\top h(x_i) = y_i : i=1,\ldots n\right]
  = \expectation\left[\left\{f(x)-f_n^*\right\}^{+}\mid  x_i,f(x_i)=y_i : i=1,\ldots n\right].
\end{align*}

However, if we condition on $h(x_i)$ rather than $w^\top h(x_i)$ in the left-hand side, then the equality is no longer true, even if the values on which we condition satisfy $w^\top h(x_i) = y_i$:
\begin{align*}
 \expectation\left[\left\{w^\top h(x)-f_n^*\right\}^{+}\mid x_i, h(x_i) : i=1,\ldots n\right]
  \ne \expectation\left[\left\{f(x)-f_n^*\right\}^{+}\mid  x_i,f(x_i)=y_i : i=1,\ldots n\right].
\end{align*}

Thus, even if we initiate optimization using EI-CF and a parallel optimization using EI with a single-output Gaussian process as described above, their acquisition functions will cease to agree once we condition on the results of an evaluation.

\section{Probability of Improvement for Composite Functions}
\label{app:picf}
In this section, we formally define the probability of improvement for composite functions (PI-CF) acquisition function and specify its implementation details used within our experimental setup.

Analogously to EI-CF, PI-CF is simply defined as the probability of improvement evaluated with respect to the implied posterior distribution on $f$ when we model $h$ as a multi-output GP:
\begin{equation*}
    \picf(x) = \prob_n\left(g(h(x))\geq f_n^* + \delta\right),
\end{equation*}
where $\prob_n$ denotes the conditional probability given the available observations at time $n$, $\{\left(x_i,h(x_i)\right)\}_{i=1}^n$, and $\delta>0$ is a parameter to be specified. As we did with $\ei$, we can express $\picf(x)$ as
\begin{equation*}
    \picf(x) = \prob_n\left(g(\mu_n(x) + C_n(x)Z)\geq f_n^* + \delta\right),
\end{equation*}
where $Z$ is a $m$-variate standard normal random vector under the time-$n$ posterior distribution. 

We can further rewrite $\picf(x)$ using an indicator function $\mathbb{I}$ as
\begin{equation*}
    \picf(x) = \expectation_n\left[\mathbb{I}\{g(\mu_n(x) + C_n(x)Z)\geq f_n^* + \delta\}\right],
\end{equation*}
which implies that $\picf$ can be computed with arbitrary precision following a Monte Carlo approach as well:
\begin{equation*}
    \picf(x) \approx \frac{1}{L}\sum_{\ell=1}^L\mathbb{I}\left\{g\left(\mu_n(x) + C_n(x)Z^{(\ell)}\right)\geq f_n^* + \delta\right\},
\end{equation*}
where $Z^{(1)},\ldots, Z^{(L)}$ are draws of an $m$-variate standard normal random vector. However, an unbiased estimator of the gradient of PI-CF cannot be computed following an analogous approach to the one used with EI-CF. In fact, $\nabla\mathbb{I}\{g(\mu_n(x) + C_n(x)Z)\geq f_n^* + \delta\}=0$ at those points for which the function $x\mapsto \mathbb{I}\{g(\mu_n(x) + C_n(x)Z)>\geq f_n^* + \delta\}$ is differentiable. Thus, even if $\nabla\mathbb{I}\{g(\mu_n(x) + C_n(x)Z)\geq f_n^* + \delta\}$ exists, in general
\begin{equation*}
    \nabla\picf(x) \neq \expectation_n\left[\nabla\mathbb{I}\{g(\mu_n(x) + C_n(x)Z)\geq f_n^* + \delta\}\right], 
\end{equation*}
unless $\nabla\picf(x)=0$. 

In our experiments, we adopt a sample average approximation \cite{kim2015guide} scheme for approximately maximizing PI-CF. At each iteration we fix $Z^{(1)},\ldots, Z^{(L)}$ and choose the next point to evaluate as 
\begin{equation*}
   x_{n+1}\in \arg\max_{x\in\domain}\frac{1}{L}\sum_{\ell=1}^L\mathbb{I}\left\{g\left(\mu_n(x) + C_n(x)Z^{(\ell)}\right)\geq f_n^* + \delta\right\},
\end{equation*}
where $L=50$ and $\delta = 0.01$. We solve the above optimization problem using the derivative-free optimization algorithm, CMA-ES \cite{hansen2016cma}.

\section{Description of Langermann and Rosenbrock Test Problems}
\label{app:test_functions}
The following pair of test problems are standard benchmarks in the global optimization literature. In this section, we describe in detail how they are adapted our setting, i.e., how we express them as composite functions.
\subsection{Langermann Function}
The Langermann function \citep{langermann_supp} is defined by
$f(x) = g(h(x))$ where
\begin{align*}
    h_j (x) &= \sum_{i=1}^d(x_i - A_{ij})^2, \ j=1,\ldots m, \\
    g(y) &= -\sum_{j=1}^m c_j\exp(-y_j/\pi)\cos(\pi y_j).
\end{align*}
In our experiment we set $d=2$, $m=5$, $c=(1, 2, 5, 2, 3)$,
\begin{equation*}
    A = \begin{pmatrix} 
3 & 5 & 2 & 1 & 7 \\
5 & 2 & 1 & 4 & 9
\end{pmatrix},
\end{equation*}
and $\domain = [0,10]^2$.

\subsection{Rosenbrock Function}
The Rosenbrock function \citep{rosenbrock_supp} is 
\begin{equation*}
    f(x) = - \sum_{j=1}^{d-1}100(x_{j+1} - x_j^2)^2 +(x_j - 1)^2
\end{equation*}
We adapt this problem to our framework by taking $d=5$ and defining $h$ and $g$ by
\begin{align*}
    h_j(x) &= x_{j+1} - x_j^2, \ j=1,\ldots, 4, \\
    h_{j+4}(x) &= x_j, \ j=1,\ldots, 4, \\
g(y) &= -\sum_{j=1}^4100y_{j}^2 + (y_{j+4}-1)^2.
\end{align*}
\section{Asymptotic Consistency of the Expected Improvement for Composite Functions}
\label{app:eicf_convergence}
\subsection{Basic Definitions and Assumptions}
In this section we prove that, under suitable conditions, the expected improvement sampling policy is asymptotically consistent in our setting. In the standard Bayesian optimization setting, this was first proved under quite general conditions by \citet{vazquez2010convergence_supp}. Later, \citet{bull2011convergence} provided convergence rates for several expected improvement-type policies both with fixed hyperparameters and hyperparameters estimated from the data in suitable way. Here, we restrict to prove asymptotic consistency, under fixed hyperparameters, following a similar approach to \citet{vazquez2010convergence_supp}. In particular, we provide a generalization of the No-Empty-Ball (NEB) condition, under which the expected improvement sampling policy is guaranteed to be asymptotically consistent in our setting. In the reminder of this work $\{x_n\}_{n\in\N}$ denotes the sequence of points at which $h$ is evaluated, which is not necessarily given by the expected improvement acquisition function, unless explicitly stated.
\begin{definition}[Generalized-No-Empty-Ball property]
\pfcomment{The property is of $K$ but then the condition is about $K_n$.  Are we then fixing a particular sampling policy?  Any sampling policy?} \racomment{Yeah, the property most hold for all sequences}
We shall say that a kernel, $K$, satisfies the Generalized-No-Empty-Ball (GNEB) property if, for all sequences $\{x_n\}_\N$  in $\domain$ and all $\tilde{x}\in\domain$, the following assertions are equivalent:
\begin{enumerate}
    \item $\tilde{x}$ is a limit point of $\{x_n\}_{n\in\N}$.
    \item There exists a subsequence of  $\{K_n(\tilde{x})\}_{n\in\N}$ converging to a singular matrix.
\end{enumerate}
\end{definition}

We highlight that, if $K$ is diagonal, i.e. if the output components are independent of each other, the GNEB property holds provided that at least one of its components satisfies the standard NEB property. 
\pfcomment{is this really iff?  Or does it just go one direction?  Seems like when you have diagonal $K$, one component satisfying the NEB property is enough to guarantee equivalence between $\tilde{x}$ being a limit point and that component of $K_n(\tilde{x})$ converging to $0$. If the other components of $K_n(\tilde{x})$ converged to something non-zero, you would still have $K_n(\tilde{x})$ converging to a singular matrix.} \racomment{Yeah, this was wrong. I think it is right now.}
In particular, the following result is a corollary of Proposition 10 in  \citet{vazquez2010convergence_supp}.
\begin{corollary}
Suppose $K$ is diagonal and at least one of its components has a spectral density whose inverse has at most polynomial growth. Then, $K$ satisfies the GNEB property. 
\end{corollary}
Thus, the GNEB property holds, in particular, if $K$ is diagonal and at least one of its components is a Mat\'ern kernel \citep{stein2012interpolation}.
\pfcomment{I believe this is true, even though I think the claim above is not iff.  You just need diagonal + each component satisfying NEB implies GNEB, which seems true.} \racomment{One component being a Matern kernel should be enough.}

Now we introduce some additional notation. We denote by $\mathcal{H}$ to the the reproducing kernel Hilbert space associated with $K$ \cite{alvarez2012kernels_supp}. As is standard in Bayesian optimization, we make a slight abuse of notation and denote by $h$ both a fixed element of $\mathcal{H}$ and a 
random function distributed according to a Gaussian process with mean $\mu$ and kernel $K$ (below we assume $\mu=0$); we shall explicitly state whenever $h$ is held fixed. As before, we denote $K_n(x,x)$ by $K_n(x)$. Finally, we make the following standing assumptions.
\begin{enumerate}
    \item $\domain$ is a compact subset of $\R^d$, for some $d\geq 1$.
    \item The prior mean function is identically $0$. \pfcomment{Would it be possible to relax this assumption about the prior mean?  Some reviewers may be fine with it, but others might complain that it is not necessarily true in some applications.  Even if we can't relax it, is there some way to comment on whether the assumption is important?}  \item $K$ is continuous, positive definite, and satisfies the GNEB property.
    \item $g:\R^m\rightarrow\R$ is continuous.
    \item For any bounded sequences $\{a_n\}_{n\in\N} \subset \R^m$ and $\{A_n\}_{n\in\N}\subset \R^{m\times m}$, $\expectation[\sup_{n}{|g(a_n + A_nZ)|}]<\infty$, where the expectation is over $Z$ and $Z$ is a $m$-dimensional standard normal random vector.
\end{enumerate}
The assumption that $g$ is continuous guarantees that $f=g\circ h$ is continuous, provided that $h$ is continuous as well. Moreover, in this case, since $\domain$ is compact, $f$ attains its maximum value in $\domain$; we shall denote this maximum value by $M$, i.e., $M = \max_{x \in \domain}f(x)$.
\subsection{Preliminary Results}
Before proving  asymptotic consistency, we prove several auxiliary results. We begin by proving that $\ei_n$ is continuous.
\begin{proposition}
\pfcomment{Instead of $\rightarrow$, it is better to use $\mapsto$ for function definitions.} \racomment{Is that standard? I've always seen $:\rightarrow$, and $\mapsto$ is used to specify the decision rule.}
For any $n\in\N$, the function $\ei_n:\domain \rightarrow \R$ defined by
\begin{equation*}
    \ei_n(x) = \expectation[\{g\left(\mu_{n}(x) + C_{n}(x)Z\right) - f_n^*\}^{+}],
\end{equation*}
where the expectation is over $Z$ and $Z$ is a $m$-dimensional standard normal random vector, is continuous.  \pfcomment{In the expectation, do we want a subscript $n$?  Or are we assuming that we always are conditioning on the time-$n$ posterior when we take an expectation.} \racomment{Yeah, since somehow we are fixing the sequence $\{x_n\}_n$ in advance, we think of $\mu_n(x)$ and $K_n(x)$ as deterministic. Is there a good way to say this without making it confusing?}
\end{proposition}
\begin{proof}
Let $\{x'_k\}_{k\in\N}\subset\domain$ be a convergent sequence with limit $x'_\infty$.
\pfcomment{Have we defined $C_n$?  Also, to say that $K$ continuous implies $\mu_n$ and $C_n$ are continuous, aren't there several steps there?  For example, there is a result we cite in an older paper in Jialei's parallel EI paper that says that the Cholesky decomposition is a continuous function.} \racomment{We defined $C_n$ in the main paper. }
Since $K$ is continuous, $\mu_n$ and $C_n$ are both continuous functions of $x$, and thus $\mu_{n}(x'_k)\rightarrow \mu_{n}(x'_\infty)$ and $C_{n}(x'_k) \rightarrow C_{n}(x'_\infty)$ as $k\rightarrow\infty$. Moreover, since $g$ is continuous too, it follows by the continuous mapping theorem \citep{billingsley2013convergence} that
\begin{equation*}
  \{g(\mu_{n}(x'_k) + C_{n}(x'_k)Z) - f_n^*\}^{+} \rightarrow \{g(\mu_{n}(x'_\infty) + C_{n}(x'_\infty)Z) - f_n^*\}^{+}  
\end{equation*}
almost surely as $k\rightarrow\infty$.

Now observe that
\begin{equation*}
     \{g(\mu_{n}(x'_k) + C_{n}(x'_k)Z) - f_n^*\}^{+} \leq \sup_k|g(\mu_{n}(x'_k) + C_{n}(x'_k)Z)| + |f_n^\star|.
\end{equation*}
Moreover, the sequences $\{\mu_{n}(x'_k)\}_{k\in\N}$ and $\{C_{n}(x'_k)\}$ are both convergent (with finite limits) and thus are bounded. Hence, the above inequality, along with assumption 5 and the dominated convergence theorem \citep{williams1991}, imply that
\begin{equation*}
  \expectation[\{g(\mu_{n}(x'_k) + C_{n}(x'_k)Z) - f_n^*\}^{+}] \rightarrow \expectation[\{g(\mu_{n}(x'_\infty) + C_{n}(x'_\infty)Z) - f_n^*\}^{+}],  
\end{equation*}
as $k\rightarrow\infty$, i.e., $\ei_{n}(x'_k)\rightarrow \ei_{n}(x'_\infty)$. Hence, $\ei_n$ is continuous.
\end{proof}
\begin{lemma}
\label{main_prop}
Let $\{x_n\}_{n\in\mathbb{N}}$ and $\{x'_n\}_{n\in\mathbb{N}}$ be two sequences in $\domain$. Assume that $\{x'_n\}_{n\in\mathbb{N}}$ is convergent, and denote by $x_\infty'$ its limit. Then, each of the following conditions implies the next one:
\begin{enumerate}
\item $x_\infty'$ is a limit point of $\{x_n\}_{n\in\mathbb{N}}$.
\item $K_n(x_n') \rightarrow 0$ as $n\rightarrow\infty$.
\item For any fixed $h\in\mathcal{H}$, $\mu_{n}(x_n') \rightarrow h(x_\infty')$ as $n\rightarrow \infty$.
\end{enumerate}
\end{lemma}
\begin{proof}
First we prove that 1 implies 2. If  $x_\infty'$ is an element of $\{x_n\}_{n\in\mathbb{N}}$, say $x_\infty' = x_{n_0}$, then, for $n\geq n_0$, we have
\begin{equation*}
    K_n(x_n') \lesssim K_{n_0}(x_n')\rightarrow K_{n_0}(x_\infty')=K_{n_0}(x_{n_0})=0,
\end{equation*}
where we use Lemma \ref{lemma_order} and the fact that $K_{n_0}$ is continuous. Now assume $x_\infty'$ is not an element of $\{x_n\}_{n\in\mathbb{N}}$. Let $\{x_{k_n}\}_{n\in\mathbb{N}}$ be a subsequence of $\{x_n\}_{n\in\mathbb{N}}$ converging to $x_\infty'$ and let $m_n = \max\{k_\ell: k_\ell \leq n\}$. Then, by Lemmas \ref{lemma_cov_no_n} and \ref{lemma_ineq_cov} we obtain
\begin{equation*}
    K_{n}(x_n') = \cov(h(x_n') - \mu_{n}(x_n')) \lesssim \cov(h(x_n') - h(x_{m_n})).
\end{equation*}
Finally, since $x_\infty'$ is not an element of $\{x_n\}_{n\in\mathbb{N}}$, $m_n \rightarrow \infty$, and it follows from the continuity of $K$ that
\begin{equation*}
   \cov(h(x_n') - h(x_{m_n}))
     = K(x_n', x_n') + K(x_{m_n}, x_{m_n}) - 2K(x_n', x_{m_n}) \rightarrow 0,
\end{equation*}
and thus $K_n(x_n')\rightarrow 0$.

Now we prove that 2 implies 3. Using the Cauchy-Schwarz inequality in $\mathcal{H}$, we obtain
\begin{equation*}
    \|h(x_n') - \mu_n(x_n')\|_2 \leq \|K_n(x_n')\|_2^\frac{1}{2}\|h\|_\mathcal{H},
\end{equation*}
 Thus,
\begin{align*}
    \|h(x_\infty') - \mu_n(x_n')\|_2 &\leq \|h(x_\infty') - h(x_n')\|_2 + |h(x_n')- \mu_n(x_n')\|_2\\
    &\leq \|h(x_\infty') - h(x_n')\|_2 + \|K_n(x_n')\|_2^\frac{1}{2}\|h\|_\mathcal{H} \rightarrow 0
\end{align*}
since $h$ is continuous.
\end{proof}

\begin{lemma}
\label{ei_converges_zero}
Let $\nu_n = \max_{x\in\domain}\ei_{n}(x)$. Then, for all $h\in\mathcal{H}$, $\lim\inf_{n\rightarrow\infty}\nu_n \rightarrow 0$.
\end{lemma}
\begin{proof}
Fix $h \in \mathcal{H}$ and let $\{x_n\}_{n\in\mathbb{N}}$ be the sequence of points generated by the expected improvement policy, i.e., $x_{n+1} \in \arg\max_{x\in\domain}\ei_n(x)$. Let $\tilde{x}$ be a limit point of $\{x_n\}_{n\in\mathbb{N}}$ and let $\{x_{k_n}\}_{n\in\mathbb{N}}$ be any subsequence converging to $\tilde{x}$. Consider the sequence $\{x_n'\}_{n\in\N}$ given by $x_n' = x_{k_\ell}$ for all $k_{\ell - 1} \leq n < k_{\ell}$, $n\in\N$. Clearly, $x_n' \rightarrow \tilde{x}$, and thus Lemma \ref{main_prop} implies that $\mu_{n}(x_n') \rightarrow h(\tilde{x})$ and $K_n(x_n') \rightarrow 0$. In particular, \pfcomment{Was this intentional that $x_{k_n-1}' = x_{k_n}$?  Or is dropping the $-1$ a typo?} \racomment{I think this is correct. Since $k_n -1$ satisfies $k_{n - 1} \leq k_n -1 < k_{n}$, by definition $x_{k_n-1}' = x_{k_n}$}
$\mu_{k_n - 1}(x_{k_n - 1}') \rightarrow h(\tilde{x})$ and $C_{k_n - 1}(x_{k_n - 1}') \rightarrow 0$, i.e., $\mu_{k_n -1}(x_{k_n}) \rightarrow h(\tilde{x})$ and $C_{k_n - 1}(x_{k_n}) \rightarrow 0$. Moreover, $\{f_n^*\}_{n\in\N}$ is a bounded increasing sequence, and thus has a finite limit, $f_{\infty}^{*}$, which satisfies $f_{\infty}^{*}\geq f(\tilde{x})$ as $\tilde{x}$ is a limit point of $\{x_n\}_{n\in\mathbb{N}}$ and $f$ is continuous. 

The sequences $\{\mu_{k_n -1}(x_{k_n})\}_{n\in\N}$ and $\{C_{k_n - 1}(x_{k_n})\}_{n\in\N}$ are convergent and thus bounded. Hence, from assumption 5 and the dominated convergence theorem we obtain that
\begin{align*}
    \expectation\left[\{g(\mu_{k_n -1}(x_{k_n}) + C_{k_n - 1}(x_{k_n})Z) - f_{k_{n-1}}^* \}^+\right] &\rightarrow \expectation\left[\{g(h(\tilde{x})) - f_\infty^*\}^+\right]\\
    &= \expectation\left[\{f(\tilde{x}) - f_\infty^*\}^+\right] = 0,
\end{align*}
but 
\begin{equation*}
   \nu_{k_{n-1}} = \expectation\left[\{g(\mu_{k_n -1}(x_{k_n}) + C_{k_n - 1}(x_{k_n})Z) - f_{k_{n-1}}^* \}^+\right], 
\end{equation*}
and thus the desired conclusion follows.
\end{proof}

\subsection{Proof of the Main Result}
We are now in position to prove that the expected improvement acquisition function is asymptotically consistent in the composite functions setting.
\begin{theorem}[Asymptotic consistency of EI-CF]
Assume that the covariance function, $K$, satisfies the GNEB property. Then, for any fixed $h\in\mathcal{H}$ and $x_\textnormal{init}\in\domain$, any (measurable) sequence $\{x_n\}_{n\in\N}$ with $x_1 = x_\textnormal{init}$ and $x_{n+1}\in\arg\max_{x\in\domain}\ei_n(x)$, $n\in\N$, satisfies $f_n^*\rightarrow M$.
\end{theorem}
\begin{proof}
First note that if $\{x_n\}_{n\in\N}$ is dense in $\domain$, then, by continuity of $f$, $f_n^*\rightarrow M$. Thus, we may assume that $\{x_n\}_{n\in\N}$ is not dense in $\domain$. For the sake of contradiction, we also assume that $f_\infty^* :=\lim_{n\rightarrow\infty} f_n^* < M$, which implies that we can find $\epsilon > 0$ such that $f_\infty^* \leq M - 2\epsilon$. 

Since $\{x_n\}_{n\in\N}$ is not dense in $\domain$, there exists $x_\star \in\domain$ that is not a limit point of $\{x_n\}_{n\in\N}$.  Applying the Cauchy-Schwarz inequality in $\mathcal{H}$, we obtain
\begin{equation*}
    \|\mu_n(x_\star) - h(x_\star)\|_2 \leq \|K_n(x_\star)\|^\frac{1}{2}\|h\|_\mathcal{H}\leq \|K(x_\star)\|_2^\frac{1}{2}\|h\|_\mathcal{H},
\end{equation*}
where in the last inequality we use that the sequence $\{K_n(x_\star)\}_{n\in\N}$ satisfies $K_{n+1}(x_\star)\lesssim K_n(x_\star) \lesssim K(x_\star)$ for all $n\in\N$. It follows that both sequences $\{\mu_n(x_\star)\}_{n\in\N}$ and $\{K_n(x_\star)\}_{n\in\N}$ are bounded and thus we can find convergent subsequences $\{\mu_{k_n}(x_\star)\}_{n\in\N}$ and $\{K_{k_n}(x_\star)\}_{n\in\N}$, say with limits $\mu_\star$ and $K_\star$, respectively. The GNEB property implies that $K_\star$ is nonsingular. Let $C_\star$ be the upper cholesky factor of $K_\star$ and let $S_\epsilon = \{y\in\R^m : M - \epsilon \leq g(y) \leq M \}$. By continuity of $g$, $S_\epsilon$ has positive Lebesgue measure, and since $K_\star$ is nonsingular, $\mu_\star + C_\star Z$ is a multivariate normal random vector with full support. Hence, $\prob(\mu_\star + C_\star Z \in S_\epsilon) > 0$. Moreover,
\begin{align*}
    \expectation\left[\{g(\mu_\star + C_\star Z) - f_\infty^*\}^+\right]&\geq\expectation[\epsilon\mathbb{I}\{\mu_\star + C_\star Z \in S_\epsilon\}] \\
    &= \epsilon \prob(\mu_\star + C_\star Z \in S_\epsilon) > 0.
\end{align*}

Finally, using Fatou's lemma we obtain
\begin{equation*}
    \liminf\limits_{n \rightarrow \infty} \expectation\left[\{g(\mu_{k_n}(x_\star) + C_{k_n}(x_\star)Z) - f_{k_n}^* \}^+\right] \geq \expectation\left[\{g(\mu_\star + C_\star Z) - f_\infty^*\}^+\right] > 0,
\end{equation*}
i.e., $\liminf\limits_{n \rightarrow \infty}\ei_{k_n}(x_\star)>0$, which contradicts lemma \ref{ei_converges_zero}.
\end{proof}
\section{Auxiliary Results}
Here we state some basic results on multi-output Gaussian processes. Most of them are simple generalizations of well-known facts for single-output Gaussian processes but are included here for completeness.
\begin{lemma}
\label{lemma_cov_no_n}
Suppose that the sequence $\{x_n\}_{n\in\N}$ is deterministic. Then, for any fixed $x\in\domain$
\begin{equation*}
   K_{n}(x) = \cov(h(x) - \mu_{n}(x)). 
\end{equation*}
\end{lemma}
\begin{proof}
This result may seem obvious at first sight, but it requires careful interpretation. By definition we have
\begin{align*}
    K_{n}(x) = \cov_n(h(x) - \mu_{n}(x)),
\end{align*}
but we claim that, indeed,
\begin{align*}
    K_{n}(x) = \cov(h(x) - \mu_{n}(x)),
\end{align*}
i.e., the same equality holds even if we do not condition on the information available at time $n$. To see this, it is enough to recall that $K_{n}(x)$ only depends on $x_1,\ldots, x_n$, but not on the values of $h$ at these points. Thus, the tower property of the expectation yields
\begin{align*}
    K_{n}(x) &= \expectation[K_{n}(x)]\\
    &= \expectation\left[\expectation_{n}\left[(h(x) - \mu_{n}(x))(h(x) - \mu_{n}(x))^\top\right]\right]\\
    &= \expectation\left[(h(x) - \mu_{n}(x))(h(x) - \mu_{n}(x))^\top\right]\\
    &= \cov(h(x) - \mu_{n}(x)),
\end{align*}
where in the last equality we use that $\expectation[h(x) - \mu_{n}(x)] = 0$, which can be verified similarly:
\begin{equation*}
   \expectation[h(x) - \mu_{n}(x)] = \expectation\left[\expectation_n[h(x) - \mu_{n}(x)] \right]
   = \expectation\left[0 \right]
   = 0.
\end{equation*}
\end{proof}
We emphasize that the sequence of points generated by the expected improvement acquisition function is deterministic once $h$ (the function to be evaluated, not the Gaussian process) is fixed and thus satisfies the conditions of lemma \ref{lemma_cov_no_n}.
\pfcomment{Should we clarify this above when there is a potential for tie-breaking when choosing the point with maximal EI?}

\begin{lemma}
\label{lemma_ineq_cov}
For any $x\in\domain$, $n\in\N$ and $i\in\{1,\ldots,n\}$,
\begin{equation*}
    \cov(h(x) - \mu_{n}(x)) \lesssim \cov(h(x) - h(x_{i})).
\end{equation*} 
\end{lemma}
\begin{proof}
By the law of total covariance, we have
\begin{align*}
    \expectation\left[\cov_n(h(x) - h(x_{i}))\right] + \cov\left(\expectation_n[h(x) - h(x_{i})]\right) = \cov(h(x) - h(x_{i})),
\end{align*}
which implies that 
\begin{align*}
    \expectation\left[\cov_n(h(x) - h(x_{i}))\right]  \lesssim \cov(h(x) - h(x_{i})),
\end{align*}
Moreover, conditioned on the information at time $n$, both $\mu_{n}(x)$ and $h(x_{i})$ are deterministic. Hence,  
\begin{equation*}
    \cov_n(h(x) - \mu_{n}(x)) = \cov_n(h(x) - h(x_{i})),
\end{equation*}
but by lemma \ref{lemma_cov_no_n} we know that $\cov_n(h(x) - \mu_{n}(x)) = \cov(h(x) - \mu_{n}(x))$, and thus $\expectation[\cov_n(h(x) - h(x_{i}))] = \cov(h(x) - \mu_{n}(x))$, which completes the proof.
\end{proof}

\pfcomment{Can we either have a proof (even if only 1 sentence) or a statement that the proof follows from Cauchy-Schwarz?  All of the other lemmas have proofs.}

\begin{lemma}
\label{lemma_order}
For any fixed $x\in\domain$ and $n\in\N$,
\begin{equation*}
     K_{n+1}(x)\lesssim K_n(x) \lesssim K(x)
\end{equation*}
\end{lemma}
\begin{proof}
Let $K_0 = K$. The standard formula for the posterior covariance matrix applied to the case where only one additional point is observed yields
\begin{equation*}
    K_{n+1}(x) = K_n(x) - K_{n}(x,x_{n+1})K_{n}(x_{n+1},x_{n+1})^{-1}K_{n}(x_{n+1}, x)
\end{equation*}
for all $n\geq 0$, from which the desired conclusion follows.
\end{proof}
\begin{lemma}
For any fixed $h\in\mathcal{H}$, $n\in\N$ and $x\in\domain$,
\begin{equation*}
    \|h(x) - \mu_n(x)\|_2 \leq \|K_{n}(x)\|_2^{\frac{1}{2}}\|h\|_{\mathcal{H}},
\end{equation*}
where $\|K_{n}(x)\|_2$ denotes the spectral norm of the matrix $K_{n}(x)$.
\end{lemma}
\putbib
\end{bibunit}
\end{document}